\documentclass{article}

\usepackage{times}
\usepackage{graphicx} % more modern
\usepackage{subfigure} 

\usepackage{natbib}

\usepackage{algorithm}
\usepackage{algorithmic}

\usepackage{hyperref}

\usepackage[accepted]{icml2015}

\usepackage{amsmath}
\usepackage{amsthm}
\usepackage{amsfonts}
\usepackage{float}

\usepackage{epstopdf}

\newtheorem{theorem}{Theorem}[section]
\newtheorem{lemma}{Lemma}[section]
\newtheorem{proposition}{Proposition}[section]

\newtheorem{remark}{Remark}[section]

\newtheorem{definition}{Definition}[section]
\theoremstyle{definition}

\newcommand{\sx}{{\bf S_x}}
\newcommand{\sy}{{\bf S_y}}
\newcommand{\sxy}{{\bf S_{xy}}}
\newcommand{\syx}{{\bf S_{yx}}}
\newcommand{\px}{{\bf \Phi}}
\newcommand{\py}{{\bf \Psi}}
\newcommand{\U}{{\bf U}}
\newcommand{\D}{{\bf D}}
\newcommand{\V}{{\bf V}}
\newcommand{\F}{{\bf F}}
\newcommand{\R}{{\bf R}}
\newcommand{\K}{{\bf K}}
\newcommand{\Q}{{\bf Q}}
\newcommand{\X}{{\bf X}}
\newcommand{\Y}{{\bf Y}}
\newcommand{\W}{{\bf W}}
\newcommand{\sxx}{{\bf S_x^{\frac{1}{2}}}}
\newcommand{\syy}{{\bf S_y^{\frac{1}{2}}}}
\newcommand{\sxxx}{{\bf S_x^{-\frac{1}{2}}}}
\newcommand{\syyy}{{\bf S_y^{-\frac{1}{2}}}}

\newcommand{\lm}{{\bf \Lambda}}

\newcommand{\tx}{\Delta\phi}
\newcommand{\ty}{\Delta\psi}
\newcommand{\ttx}{\Delta\widetilde{\phi}}
\newcommand{\tty}{\Delta\widetilde{\psi}}
\newcommand{\tpx}{\widetilde{\phi}}
\newcommand{\tpy}{\widetilde{\psi}}
\newcommand{\xnorm}[1]{{\!\|}#1{\|_x\!}}
\newcommand{\ynorm}[1]{{\!\|}#1{\|_y\!}}

\icmltitlerunning{Scalable Canonical Correlation Analysis}

\begin{document} 

\twocolumn[
\icmltitle{Finding Linear Structure in Large Datasets with \\Scalable Canonical Correlation Analysis}

% unless you've provided the [accepted] option to the icml2015
\icmlauthor{Zhuang Ma}{zhuangma@wharton.upenn.edu}
%\icmladdress{Statistics, The Wharton School, University of Pennsylvania, Philadelphia, PA 19104 U.S.A}
\icmlauthor{Yichao Lu}{yichaolu@wharton.upenn.edu}
%\icmladdress{Statistics, The Wharton School, University of Pennsylvania, Philadelphia, PA 19104 U.S.A}
\icmlauthor{Dean Foster}{dean@foster.net}
\icmladdress{Department of Statistics, The Wharton School, University of Pennsylvania, Philadelphia, PA 19104 U.S.A}
\icmlkeywords{Canonical Correlation Analysis, Stochastic Optimization, Large Scale Algorithm}

\vskip 0.3in
]

\begin{abstract} 
Canonical Correlation Analysis (CCA) is a widely used spectral technique for finding correlation structures in multi-view datasets. In this paper, we tackle the problem of large scale CCA, where classical algorithms, usually requiring computing the product of two huge matrices and huge matrix decomposition, are computationally and storage expensive. We recast CCA from a novel perspective and propose a scalable and memory efficient \textit{Augmented Approximate Gradient (AppGrad)} scheme for finding top $k$ dimensional canonical subspace which only involves  large matrix multiplying a thin matrix of width $k$ and small matrix decomposition of dimension $k\times k$. Further, \textit{AppGrad} achieves optimal storage complexity $O(k(p_1+p_2))$, compared with classical algorithms which usually require $O(p_1^2+p_2^2)$ space to store two dense whitening matrices.  The proposed scheme naturally generalizes to stochastic optimization regime, especially efficient for huge datasets where batch algorithms are prohibitive. The online property of stochastic \textit{AppGrad} is also well suited to   the streaming scenario, where data comes sequentially. To the best of our knowledge, it is the first stochastic algorithm for CCA. Experiments on four real data sets are provided to show the effectiveness of the proposed methods. 
\end{abstract} 

\section{Introduction}
\subsection{Background}
Canonical Correlation Analysis (CCA), first introduced in 1936 by \cite{hotelling36}, is a foundamental statistical tool to characterize the relationship between two multidimensional variables, which finds a wide range of applications. For example, CCA naturally fits into multi-view learning tasks and tailored to generate low dimensional feature representations using abandunt and inexpensive unlabeled datasets to supplement or refine the expensive labeled data in a semi-supervised fashion. Improved generalization accuracy has been witnessed or proved in areas such as regression \cite{kakade07}, clustering \cite{chaudhuri2009multi,blaschko2008correlational}, dimension reduction \cite{foster08, mcwilliams2013correlated}, word embeddings \cite{dhillon11,dhillon12}, etc. Besides, CCA has also been succesfully applied to genome-wide association study (GWAS) and has been shown powerful for understanding the relationship between genetic variations and phenotypes \cite{witten2009penalized, chen2012structured}.

There are various equivalent ways to define CCA and here we use the linear algebraic formulation of \cite{golub1995canonical}, which captures the very essense of the procedure, pursuing the directions of maximal correlations between two data matrices. 
\begin{definition}
For data matrices $\X\in\mathbb{R}^{n\times p_1}, \Y\in \mathbb{R}^{n\times p_2}$ Let $\sx=\X^\top\X/n, \, \sy=\Y^\top\Y/n,\, \sxy=\X^\top\Y/n$ and $p=\min\{p_1, p_2\}$. The canonical correlations $\lambda_1, \cdots, \lambda_p$ and corresponding pair of canonical vectors $\{(\phi_i, \psi_i)\}_{i=1}^p$ between $\X$ and $\Y$ are defined recursively by 
\begin{align*}
\begin{gathered}
(\phi_j, \psi_j)=\arg\max\limits_{\substack{\phi^\top\sx\phi=1, \,\psi^\top\sy\psi=1 \\ \phi^\top\sx\phi_i=0, \,\psi^\top\sy\psi_i=0, \,1\leq i\leq j-1 }} \phi^\top\sxy\psi \\
\lambda_j=\phi_j^\top\sxy\psi_j  \quad j=1, \cdots, p
\end{gathered}
\end{align*}
\end{definition}
%The following result reduces CCA to a Singular Value Decompostion (SVD) of the whitened covariance matrix \cite{golub1995canonical}.

\begin{lemma}
\label{def}
Let $\sxxx\sxy\syyy=\U\D\V^{\top}$ be the singular value decomposition. Then $\px={\bf S_x^{-\frac{1}{2}}}\U$, $\py={\bf S_y^{-\frac{1}{2}}} \V$, and $\lm=\D$ where $\px=(\phi_1, \cdots, \phi_p), \py=(\psi_1, \cdots, \psi_p)$ and $\lm=diag(\lambda_1, \cdots, \lambda_p)$. 
\end{lemma}

The identifiability of canonical vectors $(\px, \py)$ is equivalent to the identifiability of the singular vectors $(\U, \V)$. Lemma~\ref{def} implies that the leading $k$ dimensional CCA subspace can be solved by first computing the whitening matrices $\sxxx, \syyy$ and then perform a $k$-truncated SVD on the whitened covariance matrix $\sxxx\sxy\syyy$.  This classical algorithm is feasible and accurate when the data matrices are small but it can be slow and numerically unstable for large scale datasets which are common in modern natural language processing (large corpora, \citet{dhillon11, dhillon12}) and multi-view learning (abandunt and inexpensive unlabeled data, \citet{hariharanlarge}) applications. 

Throughout the paper, we call the step of orthonormalizing the columns of $\X$ and $\Y$ \textbf{whitening step}. The computational complexity of the classical algorithm is dominated by the whitening step. There are two major bottlenecks,
\begin{itemize}
\item  Huge matrix multiplication $\X^\top\X, \Y^\top\Y$ to obtain $\sx, \sy$ with computational complexity $O(np_1^2+np_2^2)$ for general dense $\X$ and $\Y$. 
\item  Large matrix decomposition to compute $\sxxx$ and $\syyy$ with computational complexity $O(p_1^3+p_2^3)$ (Even when $\X$ and $\Y$ are sparse, $\sx, \sy$ are not necessarily sparse)
\end{itemize}
\begin{remark}
The whitening step dominates the $k$-truncated SVD step because the top $k$ dimensional singular vectors can be efficiently computed by randomized SVD algorithms (see \citet{tropp11} and many others). 
%when applying power iterations, instead of directly computing $\sxxx\sxy\syyy$, we only need to multiply $\sxxx\sxy\syyy$ with a thin matrix, which can be efficiently achieved by successively multiplying each huge matrix with the thin matrix. 
\end{remark}
\begin{remark}
Another classical algorithm (built-in function in Matlab) introduced in \cite{bjorck1973numerical} uses a different way of whitening. It first carrys out a QR decomposition, $\X=\Q_x\R_x$ and $\Y=\Q_y\R_y$ and then performs a SVD on $\Q_x^\top\Q_y$, which has the same computational complexity $O(np_1^2+np_2^2)$ as the algorithm indicated by Lemma~\ref{def}. However, it is difficult to exploit sparsity in QR factorization while $\X^\top\X, \Y^\top\Y$ can be efficiently computed when $\X$ and $\Y$ are sparse. 
\end{remark}

Besides computational issues, extra $O(p_1^2+p_2^2)$ space  is necessary to store two whitening matrices $\sxxx$ and $\syyy$ (typically dense). In high dimensional applications where the number of features is huge, this can be another bottleneck considering the capacity of RAM of personal desktops (10-20 GB). In large distributed storage systems, the extra required space might incur heavy communication cost. 

Therefore, it is natural to ask: is there a scalable algorithm that avoids huge matrix decomposition and huge matrix multiplication? Is it memory efficient? Or even more ambitiously, is there an online algorithm that generates decent approximation given a fixed computational power (e.g. CPU time, FLOP)? 

\subsection{Related Work}
Scalability begins to play an increasingly important role in modern machine learning applications and draws more and more attention. Recently lots of promising progress emerged in the literature concerning with randomized algorithms for large scale matrix approximations, SVD, and Principal Component Analysis \cite{sarlos2006improved, liberty2007randomized, woolfe2008fast, tropp11}. Unfortunately, these techniques does not directly solve CCA due to the whitening step. Several authors have tried to devise a scalable CCA algorithm. \citet{avron13} proposed an efficient approach for CCA between two tall and thin matrices ($p_1, p_2\ll n$) harnessing the recently developed tools, \textit{Subsampled Randomized Hadamard Transform}, which only subsampled a small proportion of the $n$ data points to approximate the matrix product. However, when the size of the features, $p_1$ and $p_2$, are large, the sampling scheme does not work. Later, \citet{lu2014large} consider sparse design matrices and formulate CCA as iterative least squares, where in each iteration a fast regression algorithm that exploits sparsity is applied. %All current algorithms are all batch and offline algorithms, which are not feasible either when computational resources are limited or data comes as a stream.

Another related line of research considers stochastic optimization algorithms for PCA \cite{arora2012stochastic, mitliagkas2013memory, balsubramani2013fast}, which date back to \citet{oja1985stochastic}. Compared with batch algorithms, the stochastic versions empirically converge much faster with similar accuracy. Further, these stochastic algorithms can be applied to streaming setting where data comes sequentially (one pass or several pass) without being stored. As mentioned in \cite{arora2012stochastic}, stochastic optimization algorithm for CCA is more challenging and remains an open problem because of the whitening step.

\subsection{Main Contribution}
The main contribution of this paper is to directly tackle CCA as a nonconvex optimization problem and propose a novel Augmented Approximate Gradient (\textit{AppGrad}) scheme and its stochastic variant for finding the top $k$ dimensional canonical subspace. Its advantages over state-of-art CCA algorithms are three folds. \textit{Firstly}, \textit{AppGrad} scheme only involves large matrix multiplying a thin matrix of width $k$ and small matrix decomposition of dimension $k\times k$, and therefore to some extent is free from the two bottlenecks. It also benefits if $\X$ and $\Y$ are sparse while classical algorithm still needs to invert the dense matrices $\X^\top\X$ and $\Y^\top\Y$.  \textit{ Secondly}, \textit{AppGrad} achieves optimal storage complexity $O(k(p_1+p_2))$, the space necessary to store the output, compared with classical algorithms which usually require $O(p_1^2+p_2^2)$ for storing the whitening matrices. \textit{Thirdly}, the stochastic (online) variant of \textit{AppGrad} is especially efficient for large scale datasets if moderate accuracy is desired. It is well-suited to the case when computational resources are limited or data comes as a stream. To the best of our knowledge, it is the first stochastic algorithm for CCA, which partly gives an affirmative answer to a question left open in \cite{arora2012stochastic}.  

The rest of the paper is organized as follows. We introduce \textit{AppGrad} scheme and establish its convergence properties in section 2. We extend the algorithm to stochastic settings in section 3. Extensive real data experiments are presented in section 4. Concluding remarks and future work are summarized in section 5. Proof of Theorem~\ref{thm} and Proposition~\ref{fixpoint-k} are relegated to the supplementary material. 

\section{Algorithm}
For simplicity, we first focus on the leading canonical pair $(\phi_1, \psi_1)$ to motivate the proposed algorithms. Results for general scenario can be obtained in the same manner and will be briefly discussed in the later part of this section. 
\subsection{An Optimization Perspective}
Throughout the paper, we assume $\X$ and $\Y$ are of full rank. We use $\|\cdot\|$ for $L_2$ norm.  $\forall \,u\in \mathbb{R}^{p_1}, \, v\in\mathbb{R}^{p_2}$, we define $\| u\| _x=(u^\top\sx u)^{\frac{1}{2}}$ and $\| v\| _y=(v^\top\sy v)^{\frac{1}{2}}$, which are norms induced by $\X$ and $\Y$. %For the rest of the paper, we will use $(u^\top\sx u)^{\frac{1}{2}}, (v^\top\sy v)^{\frac{1}{2}}$ and their shorthands interchangeably. 

To begin with, we recast CCA as an nonconvex optimization problem \cite{golub1995canonical}. 
\begin{lemma}
$(\phi_1, \psi_1)$ is the solution of 
\begin{equation}
\begin{gathered}
\min \frac{1}{2n}\|\X\phi-\Y\psi\|^2\\
\mbox{subject\, to}\quad  \phi^\top\sx\phi=1,\,\psi^\top\sy\psi=1
\end{gathered}
\label{optimization}
\end{equation}
\label{optformu}
\end{lemma}
Although \eqref{optimization} is a nonconvex (due to the nonconvex constraint), \cite{golub1995canonical} showed that an alternating minimization strategy (Algorithm~\ref{alg:ALS}), or rather iterative least squares, actually converges to the leading canonical pair. However, each update $\phi^{t+1}={\bf S_x^{-1}}\sxy\psi^t$ is computationally intensive. Essentially, the alternating least squares acts like a second order method, which is usually recognized to be inefficient for large-scale datasets, especially when current estimate is not close enough to the optimum. Therefore, it is natural to ask: is there a valid first order method that solves \eqref{optimization}?
\begin{algorithm}[tb]
   \caption{CCA via Alternating Least Squares}
   \label{alg:ALS}
\begin{algorithmic}
   \STATE {\bfseries Input:} Data matrix  $\X\in\mathbb{R}^{n\times p_1}, \Y\in \mathbb{R}^{n\times p_2}$ and initialization $(\phi^0,\psi^0)$
   \STATE {\bf Output :}$(\phi_{\textsc{Als} }, \psi_{\textsc{Als}})$% Leading canonical pair $(\phi_1, \psi_1)$%
   \REPEAT
    \STATE  $\phi^{t+1}=\arg\min\limits_{\phi}\frac{1}{2n}\|\X\phi-\Y\psi^t\|^2={\bf S_x^{-1}}\sxy\psi^t$
    \STATE  $\phi^{t+1}=\phi^{t+1}/\|\phi^{t+1}\|_x$  %$\phi^{t+1}=\frac{\phi^{t+1}}{\|\phi^{t+1}\|_x}$ 
    \STATE  $\psi^{t+1}=\arg\min\limits_{\psi}\frac{1}{2n}\|\Y\psi-\X\phi^t\|^2={\bf S_y^{-1}}\syx\phi^t$ %\quad  $\psi^{t+1}=\frac{\psi^{t+1}}{\|\psi^{t+1}\|_x}$ 
    \STATE $\psi^{t+1}=\psi^{t+1}/\|\psi^{t+1}\|_y$
   \UNTIL{convergence}
\end{algorithmic}
\end{algorithm}
%Therefore, an scalable first order algorithm with convergence guarantees is desired but remains unclear due to the nonconvex nature of the problem. 
Heuristics borrowed from convex optimization literature give rise to a projected gradient scheme summarized in Algorithm~\ref{alg:GD}. Instead of completely solving a least squares in each iterate, a single gradient step of \eqref{optimization} is performed and then project back to the constrained domain, which avoids inverting a huge matrix. Unfortunately, the following proposition demonstrates that Algorithm~\ref{alg:GD} fails to converge to the leading canonical pair. 
\begin{algorithm}[tb]
   \caption{CCA via Naive Gradient Descent}
   \label{alg:GD}
\begin{algorithmic}
   \STATE {\bfseries Input:} Data matrix  $\X\in\mathbb{R}^{n\times p_1}, \Y\in \mathbb{R}^{n\times p_2}$, initialization $(\phi^0,\psi^0)$, step size $\eta_1,\eta_2$
   \STATE {\bf Output :} NAN (incorrect algorithm) %Leading canonical pair $(\phi_{\textsc{G} }, \psi_{\textsc{G}})$
   \REPEAT
    \STATE  $\phi^{t+1}=\phi^t-\eta_1\X^\top(\X\phi^t-\Y\psi^t)/n$ %\quad  $\phi^{t+1}=\frac{\phi^{t+1}}{\|\phi^{t+1}\|_x}$ 
    \STATE  $\phi^{t+1}=\phi^{t+1}/\|\phi^{t+1}\|_x$
    \STATE  $\psi^{t+1}=\psi^t-\eta_2\Y^\top(\Y\psi^t-\X\phi^t)/n$%\quad  $\psi^{t+1}=\frac{\psi^{t+1}}{\|\psi^{t+1}\|_x}$ 
     \STATE $\psi^{t+1}=\psi^{t+1}/\|\psi^{t+1}\|_y$
   \UNTIL{convergence}
\end{algorithmic}
\end{algorithm}
%Here is an example demonstrating that heuristics (Algorithm~\ref{alg:GD}) from convex optimization literature might not work. 
\begin{proposition}
If leading canonical correlation $\lambda_1\neq 1$ and either $\phi_1$ is not an eigenvector of $\sx$ or $\psi_1$ is not an eigenvector of $\sy$,  then $\forall \eta_1,\eta_2> 0$, the leading canonical pair $(\phi_1, \psi_1)$ is not a fixed point of the naive gradient scheme in Algorithm~\ref{alg:GD}. Therefore, the algorithm does not converge to $(\phi_1, \psi_1)$. 
\label{fixgd}
\end{proposition}
\begin{proof}[Proof of Proposition~\ref{fixgd}]
The proof is similar to the proof of Proposition~\ref{fixpoint} and we leave out the details here.
\end{proof}
The failure of Algorithm~\ref{alg:GD} is due to the nonconvex nature of \eqref{optimization}. Although every gradient step might decrease the objective function, this property no longer persists after projecting to its nonconvex domain  $\big\{(\phi, \psi)\,|\, \phi^\top\sx\phi=1,\,\psi^\top\sy\psi=1\big\}$ (the normalization step). On the contrary, decreases triggered by gradient descent is always maintained if projecting to a convex region. 
%\begin{remark}
%For all the algorithms presented in this section,  during the update of $\psi^{t+1}$, $\phi^{t}$ can be replaced by $\phi^{t+1}$, simply the difference between gradient descent and block coordinate descent. 
%\end{remark}

\subsection{\textit{AppGrad} Scheme}
As a remedy, we propose a novel Augmented Approximate Gradient (\textit{AppGrad}) scheme summarized in Algorithm~\ref{alg:RG}. It inherits the convergence guarantee of alternating least squares as well as the scalability and memory efficiency of first order methods, which only involves matrix-vector multiplication and only requires $O(p_1+p_2)$ extra space. 

\begin{algorithm}[tb]
   \caption{CCA via AppGrad}
   \label{alg:RG}
\begin{algorithmic}
   \STATE {\bfseries Input:} Data matrix  $\X\in\mathbb{R}^{n\times p_1}, \Y\in \mathbb{R}^{n\times p_2}$, initialization $(\phi^0,\psi^0, \widetilde{\phi}^0, \widetilde{\psi}^0)$, step size $\eta_1,\eta_2$
   \STATE {\bf Output:} $(\phi_{\textsc{Ag} }, \psi_{\textsc{Ag}}, \widetilde{\phi}_{\textsc{Ag}}, \widetilde{\psi}_{\textsc{Ag}})$ %and its unscaled counterparts $()$ % Leading canonical pair $(\phi_1, \psi_1)$ and its unscaled counterparts  $(\tilde{\phi}_1, \tilde{\psi}_1)=\lambda_1(\phi_1, \psi_1)$.  %
   \REPEAT
    \STATE  $\tpx^{t+1}=\tpx^t-\eta_1\X^\top(\X\tpx^t-\Y\psi^t)/n$ %\quad  $\phi^{t+1}=\frac{\tpx^{t+1}}{\xnorm{\tpx^{t+1}}}$ 
   \STATE   $\phi^{t+1}=\tpx^{t+1}/\;\xnorm{\tpx^{t+1}}$ 
    \STATE  $\tpy^{t+1}=\tpy^t-\eta_2\Y^\top(\Y\tpy^t-\X\phi^t)/n$ %\quad  $\psi^{t+1}=\frac{\tpy^{t+1}}{\ynorm{\tpy^{t+1}}}$ 
          \STATE   $\phi^{t+1}=\tpy^{t+1}/\;\ynorm{\tpy^{t+1}}$ 
   \UNTIL{convergence}
\end{algorithmic}
\end{algorithm}

\textit{AppGrad} seems unnatural at first sight but has some nice intuitions behind as we will discuss later. The differences and similarities between these algorithms are subtle but crucial. Compared with the naive gradient descent, we introduce two auxiliary variables $(\tpx^t,\tpy^t)$, an unnormalized version of $(\phi^t, \psi^t)$. During each iterate, we keep updating $\tpx^t$ and $\tpy^t$ without scaling them to have unit norm, which in turn produces the `correct' normalized counterpart, $(\phi^t, \psi^t)$. It turns out that  $(\phi_1, \psi_1, \lambda_1\phi_1, \lambda_1\psi_1)$ is a fixed point of the dynamic system $\{(\phi^t, \psi^t, \tpx^t,\tpy^t)\}_{t=0}^{\infty}$.  
\begin{proposition}
\label{fixpoint}
$\forall \,  i\leq p$, let $\widetilde{\phi}_i=\lambda_i\phi_i, \widetilde{\psi}_i=\lambda_i\psi_i$, then $(\phi_i, \psi_i, \widetilde{\phi}_i, \widetilde{\psi}_i)$ are the fixed points of \textit{AppGrad} scheme. 
\end{proposition}

To prove the proposition, we need the following lemma that characterizes the relations among some key quantities. 
\begin{lemma}
\label{decom}
$\sxy=\sx\px \lm\py^\top\sy$
\end{lemma}
\vspace{-0.5cm}
\begin{proof}[Proof of Lemma~\ref{decom}]
By Lemma~\ref{def}, $\bf{S_x^{-\frac{1}{2}}S_{xy}S_y^{-\frac{1}{2}}}=\U\D\V^\top$, where $\U=\sxx\px$, $\V=\syy\py$  and $\D=\lm$. Then we have $\sxy=\sxx\U\D\V^\top\syy=\sx\px \lm\py^\top\sy$.
\end{proof}
\begin{proof}[Proof of Proposition~\ref{fixpoint}]
Substitute $(\phi^t, \psi^t, \tpx^t,\tpy^t)=(\phi_i, \psi_i, \widetilde{\phi}_i, \widetilde{\psi}_i)$ into the iterative formula in Algorithm~\ref{alg:RG}. 
\begin{equation}
\begin{aligned}
\tpx^{t+1}&=\widetilde{\phi}_i-\eta_1(\sx\widetilde{\phi}_i-\sxy\psi_i)\\
&=\widetilde{\phi}_i-\eta_1(\sx\widetilde{\phi}_i-\sx\px \lm\py^\top\sy\psi_i)\\
&=\widetilde{\phi}_i-\eta_1(\sx\widetilde{\phi}_i-\lambda_i\sx\phi_i)\\
&=\widetilde{\phi}_i
\end{aligned}
\nonumber
\end{equation}
The second equality is direct application of Lemma~\ref{decom}. The third equality is due to the fact that $\py^\top\sy\py=I_p$. Then,
$$\phi^{t+1}=\widetilde{\phi}_i/ \|\widetilde{\phi}_i\|_x=\widetilde{\phi}_i/ \lambda_i=\phi_i$$
Therefore $(\tpx^{t+1}, \phi^{t+1})=(\tpx^{t}, \phi^{t})=(\widetilde{\phi}_i, \phi_i)$.
A symmetric argument will show that $(\tpy^{t+1}, \psi^{t+1})=(\tpy^{t}, \psi^{t})=(\widetilde{\psi}_i, \psi_i)$, which completes the proof. 
\end{proof}

The connection between \textit{AppGrad} and alternating minimization strategy is not instaneous. Intuitively, when $(\phi^t, \psi^t)$ is not close to $(\phi_1, \psi_1)$, solving the least squares completely as carried out in Algorithm~\ref{alg:ALS} is a waste of computational power (informally by regarding it as a second order method, the Newton Step has fast convergence only when current estimate is close to the optimum). Instead of solving a sequence of possibly unrelevant least squares, the following lemma shows that \textit{AppGrad} directly targets at the least squares that involves the leading canonical pair.

\begin{lemma} 
Let $(\phi_1, \psi_1)$ be the leading canonical pair and $(\widetilde{\phi}_1, \widetilde{\psi}_1)=\lambda_1(\phi_1, \psi_1)$. Then,
\begin{equation}
\begin{gathered}
\widetilde{\phi}_1=\arg\min\limits_{\phi}\frac{1}{2n} \|\X\phi-\Y\psi_1\|^2\\
\widetilde{\psi}_1=\arg\min\limits_{\psi}\frac{1}{2n} \|\Y\psi-\X\phi_1\|^2
\end{gathered}
\label{regression}
\end{equation}
\label{relationship}
\end{lemma}
\vspace{-0.5cm}
\begin{proof}[Proof of Lemma~\ref{relationship}] Let $\phi^*=\arg\min\limits_{\phi}\frac{1}{2n} \|\X\phi-\Y\psi_1\|^2$, by optimality condition, $\sx\phi^*=\sxy\psi_1$. Apply Lemma~\ref{decom}, 
$$\phi^*=\sx^{-1}\sx\px \lm\py^\top\sy\psi_1=\lambda_1\phi_1=\widetilde{\phi}_1$$
Similar argument gives $\psi^*=\widetilde{\psi}_1$
\end{proof}

Lemma~\ref{relationship} characterizes the relationship between leading canonical pair $(\phi_1, \psi_1)$ and its unnormalized counterpart $(\widetilde{\phi}_1, \widetilde{\psi}_1)$, which sheds some insight on how \textit{AppGrad} works. The intuition is that $(\phi^t, \psi^t)$ and $(\tpx^{t}, \tpy^{t})$ are current estimations of $(\phi_1, \psi_1)$ and $(\widetilde{\phi}_1, \widetilde{\psi}_1)$, and the updates of $(\tpx^{t+1}, \tpy^{t+1})$ in Algorithm~\ref{alg:RG} are actually gradient steps of the least squares in \eqref{regression}, with the unknown truth $(\phi_1, \psi_1)$ approximated by $(\phi^t, \psi^t)$. In terms of mathematics, 
\begin{equation}
\begin{aligned}
\widetilde{\phi}^{t+1}&=\tpx^t-\eta_1\X^\top(\X\tpx^t-\Y\psi^t)/n\\
&\approx \tpx^t-\eta_1\X^\top(\X\tpx^t-\Y\psi_1)/n\\
&=\widetilde{\phi}^{t}-\eta_1\nabla_{\phi}\frac{1}{2n} \|\X\phi-\Y\psi_1\|^2|_{\phi=\widetilde{\phi}^{t}}
\end{aligned}
\label{approx}
\end{equation}
The normalization step in Algorithm~\ref{alg:RG} corresponds to generating new approximations of $(\phi_1, \psi_1)$, namely $(\phi^{t+1}, \psi^{t+1})$, using the updated $(\tpx^{t+1}, \tpy^{t+1})$ through the relationship $(\phi_1, \psi_1)=(\widetilde{\phi}_1/\|\widetilde{\phi}_1\|_x, \,\widetilde{\psi}_1/\|\widetilde{\psi}_1\|_y)$. Therefore, one can interpret \textit{AppGrad} as approximate gradient scheme for solving \eqref{regression}. When $(\tpx^t, \tpy^t)$ converge to $(\tpx_1, \tpy_1)$, its scaled version $(\phi^{t}, \psi^{t})$ converge to the leading canonical pair $(\phi_1, \psi_1)$. 

The following theorem shows that when the estimates enter a neighborhood of the true canonical pair, \textit{AppGrad} is contractive. Define the error metric $e_t=\|\ttx^t\|^2+\|\tty^t\|^2$ where $\ttx^t=\tpx^t-\tpx_1, \tty^t=\tpy^t-\tpy_1$.
%$e_t=\min\{\|\tpx^t-\tpx_1\|^2, \|\tpx^t+\tpx_1\|^2\}+\min\{\|\tpy^t-\tpy_1\|^2, \|\tpy^t+\tpy_1\|^2\}$. 

\begin{theorem}
Assume $\lambda_1>\lambda_2$, $\exists\, L_1, L_2\geq 1$ such that $\lambda_{max}(\sx), \lambda_{max}(\sy)\leq L_1$ and $\lambda_{min}(\sx), \lambda_{min}(\sy)\geq  L_2^{-1}$, where $\lambda_{min}(\cdot), \lambda_{max}(\cdot)$ denote smallest and largest eigenvalues. If $e_0< 2(\lambda_1^2-\lambda_2^2)/L_1$ and set $\eta_1=\eta_2=\eta=\delta/6L_1$, then AppGrad achieves linear convergence such that $\forall\, t\in\mathbb{N}_+$
\begin{equation}
e_t\leq  \Big(1-\frac{\delta^2}{6L_1L_2}\Big)^t e_0
\nonumber
\end{equation}
where $\delta=1-\Big(1-\frac{2(\lambda_1^2-\lambda_2^2)-L_1e_0}{2\lambda_1^2}\Big)^\frac{1}{2}>0$
\label{thm}
\end{theorem}
\begin{remark}
The theorem reveals that the larger is the eigengap $\lambda_1-\lambda_2$, the broader is the contraction region. We didn't try to optimize the conditions above and empirically as shown in the experiments, a randomized initialization always suffices to capture most of the correlations. 
\end{remark}

\subsection{General Rank-$k$ Case}
\begin{algorithm}[tb]
   \caption{CCA via \textit{AppGrad} (Rank-$k$)}
   \label{alg:RG-rank-k}
\begin{algorithmic}
   \STATE {\bfseries Input:} Data matrix  $\X\in\mathbb{R}^{n\times p_1}, \Y\in \mathbb{R}^{n\times p_2}$, initialization $(\px^0,\py^0, \widetilde{\px}^0, \widetilde{\py}^0)$, step size $\eta_1,\eta_2$
   \STATE {\bf Output :} $(\px_{\textsc{Ag}}, \py_{\textsc{Ag}}, \widetilde{\px}_{\textsc{Ag}}, \widetilde{\py}_{\textsc{Ag}})$ %Top $k$ dimensional canonical vectors $(\px_k, \py_k)$ and its unscaled counterparts $(\widetilde{\px}_{k}, \widetilde{\py}_{k})=(\px_k, \py_k)\Lambda_k$ % and its unscaled counterparts $(\widetilde{\px}_{\textsc{Ag}}, \widetilde{\py}_{\textsc{Ag}})$ 
   \REPEAT
    \STATE  $\widetilde{\px}^{t+1}=\widetilde{\px}^t-\eta_1\X^\top(\X\widetilde{\px}^t-\Y\py^t)/n$
    \STATE SVD: $(\widetilde{\px}^{t+1})^\top\sx\widetilde{\px}^{t+1}={\bf U}_x\D_x{\bf U}_x^\top$
    \STATE $\px^{t+1}=\widetilde{\px}^{t+1}{\bf U}_x\D_x^{-\frac{1}{2}}\U_x^\top$
    \STATE  $\widetilde{\py}^{t+1}=\widetilde{\py}^t-\eta_2\Y^\top(\Y\widetilde{\py}^t-\X\px^t)/n$
%     \STATE QR Decomposition: $\Y\widetilde{\px}^{t+1}=\Q_y\R_y$
  %  \STATE Cholesky: $(\widetilde{\py}^{t+1})^T\Y^T\Y\widetilde{\py}^{t+1}={\bf L}_y^T{\bf L}_y$
   \STATE SVD: $(\widetilde{\py}^{t+1})^\top\sy\widetilde{\py}^{t+1}={\bf U}_y\D_y{\bf U}_y^\top$
   \STATE $\py^{t+1}=\widetilde{\py}^{t+1}{\bf U}_y\D_y^{-\frac{1}{2}}\U_y^\top$
%     \STATE $\py^{t+1}=\widetilde{\py}^{t+1}{\bf L}_y^{-1}$
   \UNTIL{convergence}
\end{algorithmic}
\end{algorithm}
Following the spirit of rank-one case, \textit{AppGrad} can be easily generalized to compute the top $k$ dimesional canonical subspace as summarized in Algorithm~\ref{alg:RG-rank-k}. The only difference is that the original scalar normalization is replaced by its matrix counterpart, that is to multiply the inverse of the square root matrix $\px^{t+1}=\widetilde{\px}^{t+1}{\bf U}_x\D_x^{-\frac{1}{2}}\U_x^\top$, ensuring that  $(\px^{t+1})^\top\X^\top\X\px^{t+1}={\bf I}_k$. %SVD is used to find the square root of the matrix $(\widetilde{\px}^{t+1})^T\X^T\X\widetilde{\px}^{t+1}$, namely $\U_x\D_x^{\frac{1}{2}}\U_x^T$. 

Notice that the gradient step only involves a large matrix multiplying a thin matrix of width $k$ and the SVD  is  performed on a small $k\times k$ matrix. Therefore, the computational complexity per iteration is dominated by the gradient step, of order $O(n(p_1+p_2)k)$. The cost will be further reduced when the data matrices $\X, \Y$ are sparse. 

Compared with classical spectral agorithm which first whitens the data matrices and then performs a SVD on the whitened covariance matrix, \textit{AppGrad} actually merges these two steps together. This is the key of its efficiency. In a high level, whitening the whole data matrix is not necessary and we only want to whiten the directions that contain the leading CCA subspace. However, these directions are unknown and therefore for two-step procedures, whitening the whole data matrix is unavoidable. Instead, \textit{AppGrad} tries to identify (gradient step) and whiten (normalization step) these directions simultaneously. In this way, every normalization step is only performed on the potential $k$ dimensional target CCA subspace and therefore only deals with a small $k\times k$ matrix. 

Parallel results of Lemma~\ref{optformu}, Proposition~\ref{fixgd}, Proposition~\ref{fixpoint}, Lemma~\ref{relationship} for this general scenario can be established in a similar manner. Here, to make Algorithm~\ref{alg:RG-rank-k} more clear, we state the fixed point result of which the proof is similar to Proposition~\ref{fixpoint}. 
\begin{proposition}
\label{fixpoint-k}
Let $\lm_k=diag(\lambda_1, \cdots, \lambda_k)$ be the diagonal matrix of top $k$ canonical correlations and let $\px_k=(\phi_1, \cdots, \phi_k), \py_k=(\phi_1, \cdots, \phi_k)$ be the top $k$ CCA vectors. Also denote $\widetilde{\px}_k=\px_k\lm_k$ and $\widetilde{\py}_k=\py_k\lm_k$. Then for any $k\times k$ orthogonal matrix $\Q$, $(\px_k, \py_k, \widetilde{\px}_k,  \widetilde{\py}_k)\Q$ is a fixed point of \textit{AppGrad} scheme. 
\end{proposition}
The top $k$ dimensional canonical subspace is identifiable up to a rotation matrix and Proposition \ref{fixpoint-k} shows that every optimum is a fixed point of \textit{AppGrad} scheme. 
\subsection{Kernelization}
Sometimes CCA is restricted because of its linearity and kernel CCA offers an alternative by projecting data into a high dimensional feature space. In this section, we show that \textit{AppGrad} works for kernel CCA as well. Let $K_{\mathcal{X}}(\cdot, \cdot)$ and $K_{\mathcal{Y}}(\cdot, \cdot)$ be Mercer kernels, then there exists feature mappings $f_{\mathcal{X}}:\mathcal{X}\rightarrow \mathcal{F}_{\mathcal{X}}$ and $f_{\mathcal{Y}}:\mathcal{Y}\rightarrow \mathcal{F}_{\mathcal{Y}}$ such that $K_{\mathcal{X}}(x_i, x_j)=\left\langle f_{\mathcal{X}}(x_i), f_{\mathcal{X}}(x_j)\right\rangle$ and $K_{\mathcal{Y}}(y_i, y_j)=\left\langle f_{\mathcal{Y}}(y_i), f_{\mathcal{Y}}(y_j)\right\rangle$. Let $\F_\mathcal{X}=(f_\mathcal{X}(x_1), \cdots, f_\mathcal{X}(x_n))^\top$ and $\F_\mathcal{Y}=(f_\mathcal{Y}(y_1), \cdots, f_\mathcal{Y}(y_n))^\top$ be the compact representation of the objects in the possibly infinite dimensional feature space. Since the top $k$ dimensional canonical vectors lie in the space spaned by the features, say $\px_k=\F_{\mathcal{X}}^\top\W_{\mathcal{X}}$ and $\py_k=\F_{\mathcal{Y}}^\top\W_{\mathcal{Y}}$ for some $\W_{\mathcal{X}}, \W_{\mathcal{Y}}\in \mathbb{R}^{n\times k}$. Let $\K_\mathcal{X}=\F_{\mathcal{X}}\F_{\mathcal{X}}^\top, \K_\mathcal{Y}=\F_{\mathcal{Y}}\F_{\mathcal{Y}}^\top$ be the Gram matrices. Similar to Lemma~\ref{optformu}, kernel CCA can be formulated as
\begin{equation}
\begin{gathered}
\arg\max\limits_{\W_{\mathcal{X}}, \W_{\mathcal{Y}}} \quad \|\K_{\mathcal{X}}\W_{\mathcal{X}}-\K_{\mathcal{Y}}\W_{\mathcal{Y}}\|_F^2\\
\mbox{subject\, to} \; \W_{\mathcal{X}}^\top\K_{\mathcal{X}}\K_{\mathcal{X}}\W_{\mathcal{X}}={\bf I}_k \quad \W_{\mathcal{Y}}^\top\K_{\mathcal{Y}}\K_{\mathcal{Y}}\W_{\mathcal{Y}}={\bf I}_k
\end{gathered}
\nonumber
\end{equation}
Following the same logic as Proposition~\ref{fixpoint-k}, a similar fixed point result can be proved. Therefore, Algorithm~\ref{alg:RG-rank-k} can be directly applied to compute $\W_{\mathcal{X}}, \W_{\mathcal{Y}}$ by simply replacing $\X, \Y$ with $\K_\mathcal{X}, \K_\mathcal{Y}$.

\section{Stochastic \textit{AppGrad}}
Recently, there is a growing interest in stochastic optimization which is shown to have better performance for large-scale learning problems \cite{bousquet2008tradeoffs, bottou10}. Especially in the so-called `data laden regime', where data is abundant and the bottleneck is runtime, stochastic optimization dominate batch algorithms both empirically and theoretically. Given these advantages, lots of efforts have been spent on developing stochastic algorithms for principal component analysis \cite{oja1985stochastic, arora2012stochastic, mitliagkas2013memory, balsubramani2013fast}. Despite promising progress in PCA, as mentioned in \cite{arora2012stochastic}, stochastic CCA is more challenging and remains an open problem due to the whitening step. 

\begin{algorithm}[tb]
   \caption{CCA via Stochastic \textit{AppGrad} (Rank-$k$)}
      \label{alg:SGD}
\begin{algorithmic}
   \STATE {\bfseries Input:} Data matrix  $\X\in\mathbb{R}^{n\times p_1}, \Y\in \mathbb{R}^{n\times p_2}$, initialization $(\px^0,\py^0, \widetilde{\px}^0, \widetilde{\py}^0)$, step size $\eta_{1t},\eta_{2t}$, minibatch size $m$
   \STATE {\bf Output :} $(\px_{\textsc{Sag} }, \py_{\textsc{Sag}}, \widetilde{\px}_{\textsc{Sag}}, \widetilde{\py}_{\textsc{Sag}})$ %Top $k$ dimensional canonical vectors $(\px_k, \px_y)$ and its unscaled counterparts $(\widetilde{\px}_{k}, \widetilde{\py}_{k})=(\px_k, \px_y)\Lambda_k$ 
   %Leading canonical pair $(\px_{\textsc{Sag} }, \py_{\textsc{Sag}})$ and its unscaled counterparts $(\widetilde{\px}_{\textsc{Sag}}, \widetilde{\py}_{\textsc{Sag}})$ 
   \REPEAT
    \STATE Randomly pick a subset $\mathcal{I}\subset\{1, 2, \cdots, n\}$ of size $m$
    \STATE  $\widetilde{\px}^{t+1}=\widetilde{\px}^t-\eta_{1t}\X_{\mathcal{I}}^\top(\X_{\mathcal{I}}\widetilde{\px}^t-\Y_{\mathcal{I}}\py^t)/m$
    \STATE SVD: $(\widetilde{\px}^{t+1})^\top(\frac{1}{m}\X_{\mathcal{I}}^\top\X_{\mathcal{I}})\widetilde{\px}^{t+1}={\bf U}_x^\top\D_x{\bf U}_x$
    \STATE $\px^{t+1}=\widetilde{\px}^{t+1}{\bf U}_x^\top\D_x^{-\frac{1}{2}}{\bf U}_x$
    \STATE  $\widetilde{\py}^{t+1}=\widetilde{\py}^t-\eta_{2t}\Y_{\mathcal{I}}^\top(\Y_{\mathcal{I}}\widetilde{\py}^t-\X_{\mathcal{I}}\px^t)/m$
%     \STATE QR Decomposition: $\Y\widetilde{\px}^{t+1}=\Q_y\R_y$
    \STATE SVD: $(\widetilde{\py}^{t+1})^\top(\frac{1}{m}\Y_{\mathcal{I}}^\top\Y_{\mathcal{I}})\widetilde{\py}^{t+1}={\bf U}_y^\top\D_y{\bf U}_y$
     \STATE $\py^{t+1}=\widetilde{\py}^{t+1}{\bf U}_y^\top\D_y^{-\frac{1}{2}}{\bf U}_y$
   \UNTIL{convergence}
\end{algorithmic}
\end{algorithm}

As a gradient scheme, \textit{AppGrad} naturally generalizes to the stochastic regime and we summarize in Algorithm~\ref{alg:SGD}. Compared with the batch version, only a small subset of samples are used to compute the gradient, which reduces the computational cost per iteration from $O(n(p_1+p_2)k)$ to $O(m(p_1+p_2)k)$ ($m=|\mathcal{I}|$ is the size of the minibatch). Empirically, this makes stochastic \textit{AppGrad} much faster than the batch version as we will see in the experiments. Also, for large scale applications when fully calculating the CCA subspace is prohibitive, stochastic \textit{AppGrad} can generate a decent approximation given a fixed computational power, while other algorithms only give a one-shot estimate after the whole procedure is carried out completely. Moreover, when there is a generative model, as shown in \cite{bousquet2008tradeoffs}, due to the tradeoff between statistical and numerical accuracy, fully solving an empirical risk minimization is unnecessary since the statistical error will finally dominate. On the contrary, stochastic optimization directly tackles the problem in the population level and therefore is more statistically efficient. 

It is worth mentioning that the normalization step is accomplished using a sampled Gram matrix $\frac{1}{m}\X_{\mathcal{I}}^\top\X_{\mathcal{I}}$ and $\frac{1}{m}\Y_{\mathcal{I}}^\top\Y_{\mathcal{I}}$. A key observation is that when $m\in O(k)$, $(\widetilde{\px}^{t+1})^\top(\frac{1}{m}\X_{\mathcal{I}}^\top\X_{\mathcal{I}})\widetilde{\px}^{t+1}\approx (\widetilde{\px}^{t+1})^\top(\frac{1}{m}\X^\top\X)\widetilde{\px}^{t+1}$ using standard concentration inequality, because the matrix we want to approximate $(\widetilde{\px}^{t+1})^\top(\frac{1}{m}\X^\top\X)\widetilde{\px}^{t+1}$ is a $k\times k$ matrix, while generally $O(p)$ sample is needed to have $\frac{1}{m}\X_{\mathcal{I}}^\top\X_{\mathcal{I}}\approx \frac{1}{n}\X^\top\X$. As we have argued in previous section, this bonus is a byproduct of the fact that \textit{AppGrad} tries to identify and whiten the directions that contains the CCA subspace simultaneously, or else $O(p)$ samples are necessary for whitening the whole data matrices.

\section{Experiments}
In this section, we present experiments on four real datasets to evaluate the effectiveness of the proposed algorithms for computing the top 20 ($k$=20) dimensional canonical subspace. A short summary of the datasets is in Table~\ref{datasets}. 

\begin{table*}[t]
\caption{Brief Summary of Datasets}
\vskip 0.15in
\begin{center}
%\begin{small}
\begin{sc}
\begin{tabular}{|c|c|c|c|c|}
\hline
Datasets & Description & $p_1$&$p_2$&$n$\\
\hline
%\abovespace
Mediamill &Image and its labels & $100$ & $120$& $30,000$\\

Mnist&left and right halves of images& $392$ & $392$ &$60,000$\\

Penn Tree Bank&Word Co-ocurrance& $10,000$&$10,000$&$500,000$\\

%\belowspace
URL Reputation&Host and lexical based features& $100,000$&$100,000$&$1,000,000$\\
\hline
\end{tabular}
\end{sc}
%\end{small}
\end{center}
\label{datasets}
\vskip -0.1in
\end{table*}

\textbf{Mediamill} is an annotated video dataset from the Mediamill Challenge \cite{snoek2006challenge}. Each image is a representative keyframe of a video shot annotated with 101 labels and consists of 120 features. CCA is performed to explore the correlation structure between the images and its labels. 

%\subsubsection{Mnist}
\textbf{MNIST} is a database of handwritten digits. CCA is used to learn correlated representations between the left and right halves of the images. 
%Each image has a width and height of 28 pixels and therefore, both views are 392 dimensional features. 

%\subsubsection{Penn Tree Bank Word Co-ocurrence}
\textbf{Penn Tree Bank} dataset is extracted from Wall Street Journal, which consists of $1.17$ million tokens and a vocabulary size of $43, 000$ \cite{lamar10}. CCA has been successfully used on this dataset to build low dimensional word embeddings \cite{dhillon11,dhillon12}. The task here is a CCA between words and their context. We only consider the 10, 000 most frequent words to avoid sample sparsity.
%The rows of $\X$ matrix are indicator vectors of the current word and the rows of $\Y$ are indicators of the word behind. We only consider the 10, 000 most frequent words to avoid sample sparsity.
% $(p_1=p_2=10,000)$. 
%Both the training set and testing set contains roughly $500, 000$ samples. 
%Therefore, $\X$ and $\Y$ are very sparse. 

%\subsubsection{URL Features}
\textbf{URL Reputation} dataset \cite{ma09} is extracted from UCI machine learning repository. The dataset contains 2.4 million URLs each represented by 3.2 million features. For simplicity we only use the first 2 million samples.  $38\%$ of the features are host based features like WHOIS info, IP prefix and $62\%$ are lexical based features like Hostname and Primary domain. We run a CCA between a subset of host based features and a subset of lexical based features.
% In both training and testing set, $\X$ and $\Y$ are of size $1,000,000\times 100,000$. 

\subsection{Implementations}
\textbf{Evaluation Criterion}: The evaluation criterion we use for the first three datasets (Mediamill, MNIST, Penn Tree Bank) is \textit{Proportions of Correlations Captured} (PCC). To introduce this term, we first define \textit{Total Correlations Captured} (TCC) between two matrices to be the sum of their canonical correlations as defined in Lemma~\ref{def}. Then, for estimated top $k$ dimensional canonical subspace $\widehat{\px}_k, \widehat{\py}_k$ and true leading $k$ dimensional CCA subspace $\px_k, \py_k$, PCC is defined as 
\begin{equation}
\mbox{PCC}=\frac{\mbox{TCC}(\X\widehat{\px}_k, \Y\widehat{\py}_k)}{\mbox{TCC}(\X\px_k, \Y\py_k)}
\nonumber
\end{equation}
Intuitively PCC characterizes the proportion of correlations captured by certain algorithm compared with the true CCA subspace. Therefore, the higher is PCC the better is the estimated CCA subspace. This criterion has two major advantages over subspace distance $\|P_{\widehat{\px}_k}- P_{\px_k}\|$ ($P_\Omega$ is projection matrix of the column space of $\Omega$). First, it is more natural and relevant considering that the goal of CCA is to capture most correlations between two data matrices. Second, when the eigengap $\Delta\lambda=\lambda_k-\lambda_{k+1}$ is not big enough, the top $k$ dimensional CCA subspace is ill posed while the correlations captured is well defined. 

We use the output of the standard spectral algorithms as the truth $(\px_k, \py_k)$ to calculate the denominator of PCC. However, for URL Reputation dataset, the number of samples and features are too large for the algorithm to compute the true CCA subspace in a reasonable amount of time and instead we only compare the numerator $\mbox{TCC}(\X\widehat{\px}_k, \Y\widehat{\py}_k)$ (monotone w.r.t. PCC) for different algorithms.

\textbf{Initialization} We initialize $(\px^0,\py^0)$ by first drawing $i.i.d.$ samples from standard Gaussian distribution and then normalize such that $(\px^0)^\top\sx\px^0=I_k$ and $(\py^0)^\top\sy\py^0=I_k$

\textbf{Stepsize} For both \textit{AppGrad} and stochastic \textit{AppGrad}, a small part of the training set is held out and cross-validation is used to choose the step size adaptively. 

\textbf{Regularization} For all the algorithms, a little regularization is added for numerical stability which means we replace Gram matrix $\X^\top\X$ with $\X^\top\X+\lambda {\bf I}$ for some small positive $\lambda$. 

\textbf{Oversampling} Oversampling means when aiming for top $k$ dimensional subspace, people usually computes top $k+l$ dimesional subspace from which a best $k$ diemsional subspace is extracted. In practice, $l=5\sim 10$ suffices to improve the performance. We only do a oversampling of 5 in the URL dataset. 

\subsection{Summary of Results}
For the first three datasets (Mediamill, MNIST, Penn Tree Bank), both in-sample and out-of-sample PCC are computed for \textit{AppGrad} and Stochastic \textit{AppGrad} as summarized in Figure\ref{fig:figure}. As you can see, both algorithms nearly capture most of the correlations compared with the true CCA subspace and stochastic \textit{AppGrad} consistently achieves same PCC with much less computational cost than its batch version. Moreover, the larger is the size of the data, the bigger advantage will stochastic \textit{AppGrad} obtain. One thing to notice is that, as revealed in Mediamill dataset, out-of-sample PCC is not necessarily less than in-sample PCC because both denominator and numerator will change on the hold out set. 
\begin{figure*}
\centering
\subfigure{%
\includegraphics[width=0.65\columnwidth]{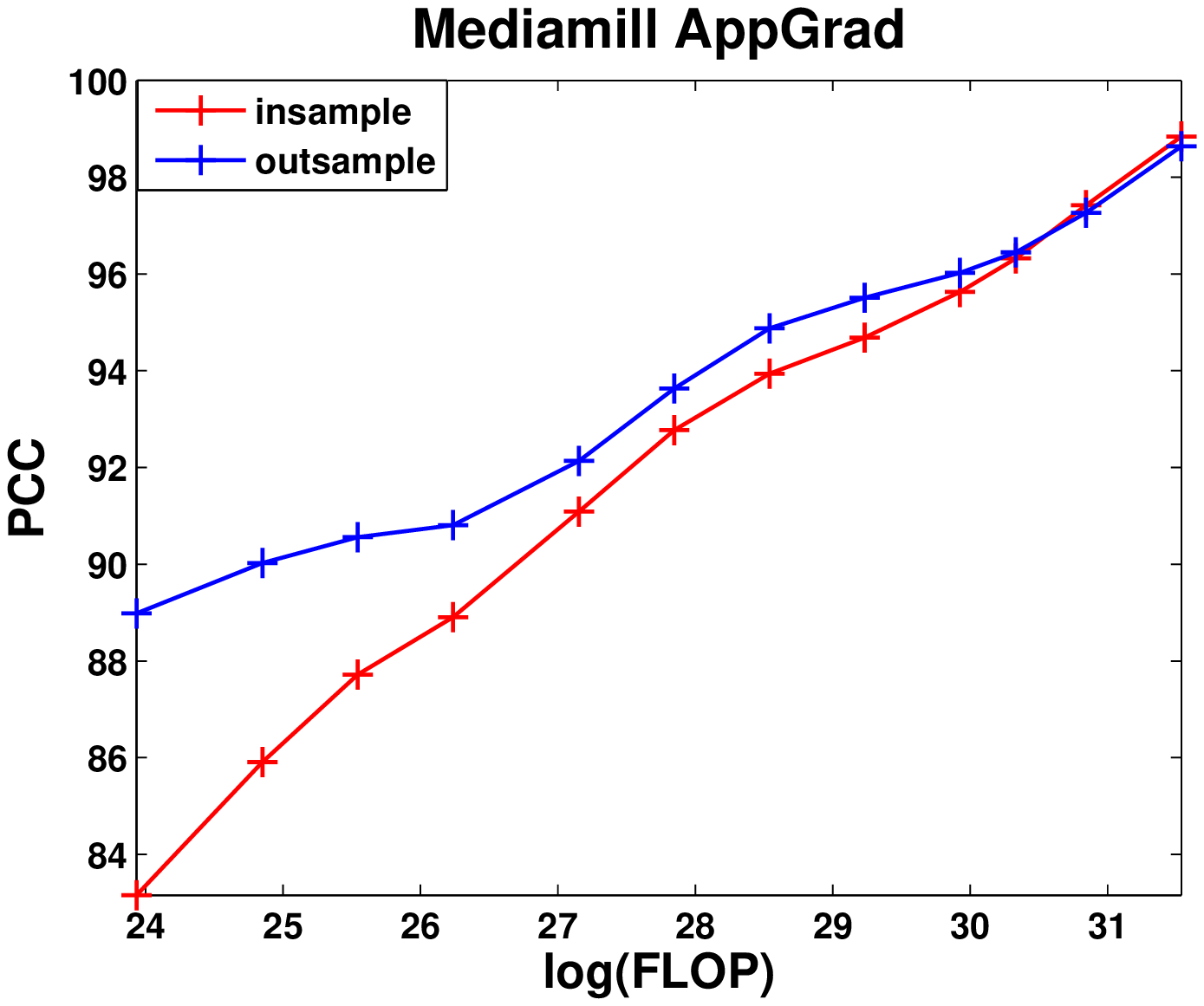}
}
\subfigure{%
\includegraphics[width=0.65\columnwidth]{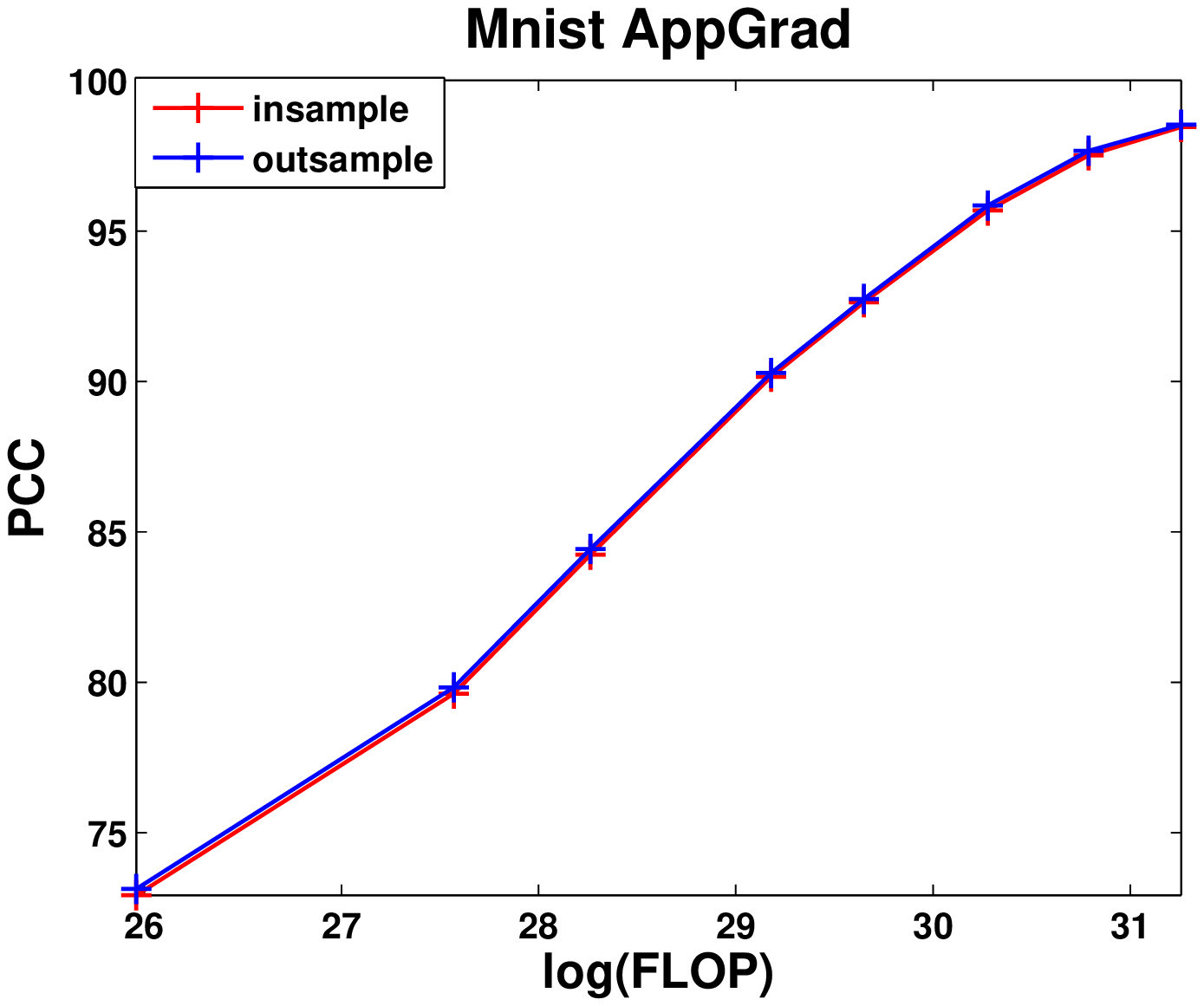}
}
\subfigure{%
\includegraphics[width=0.65\columnwidth]{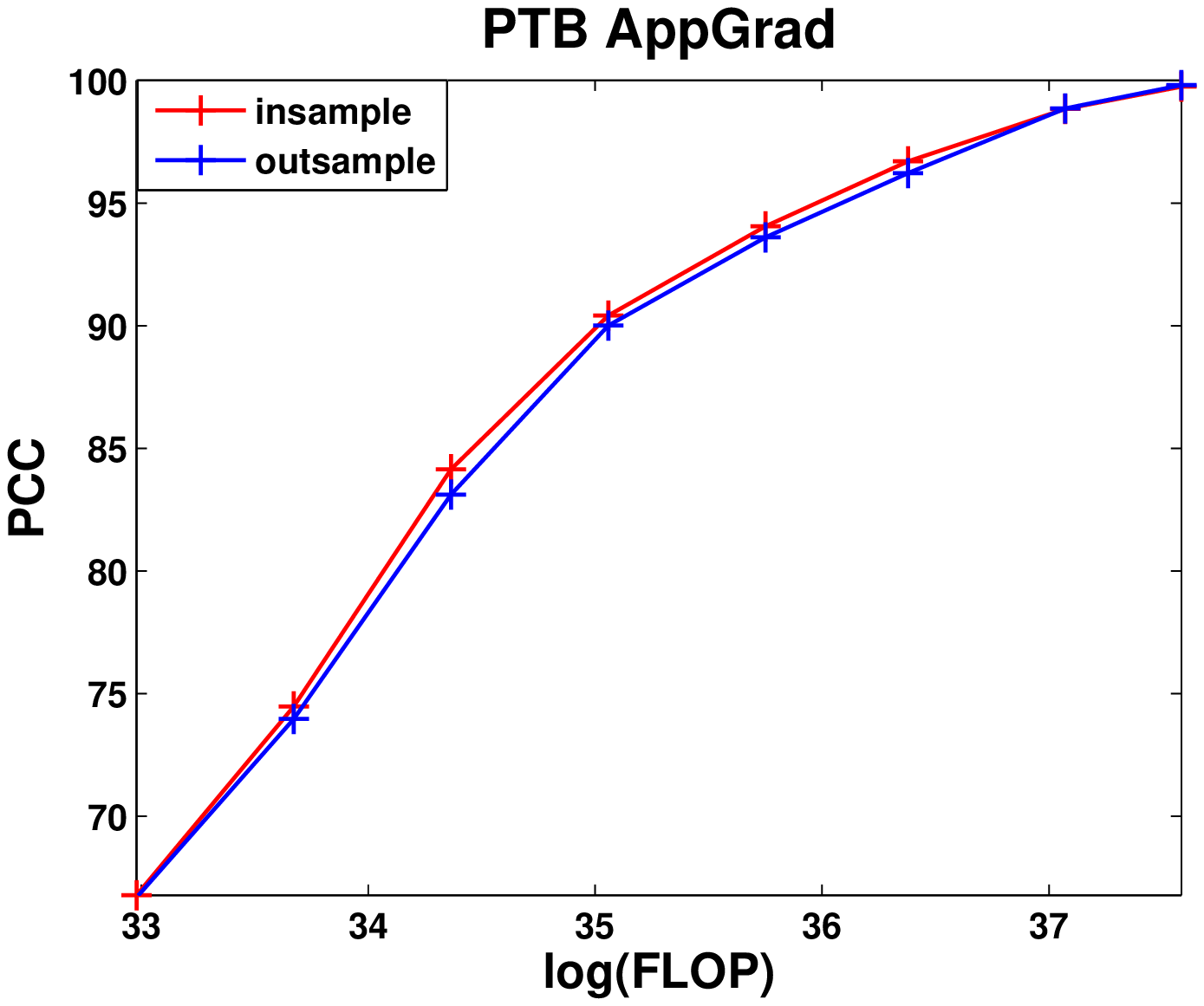}
}

\subfigure{%
\includegraphics[width=0.65\columnwidth]{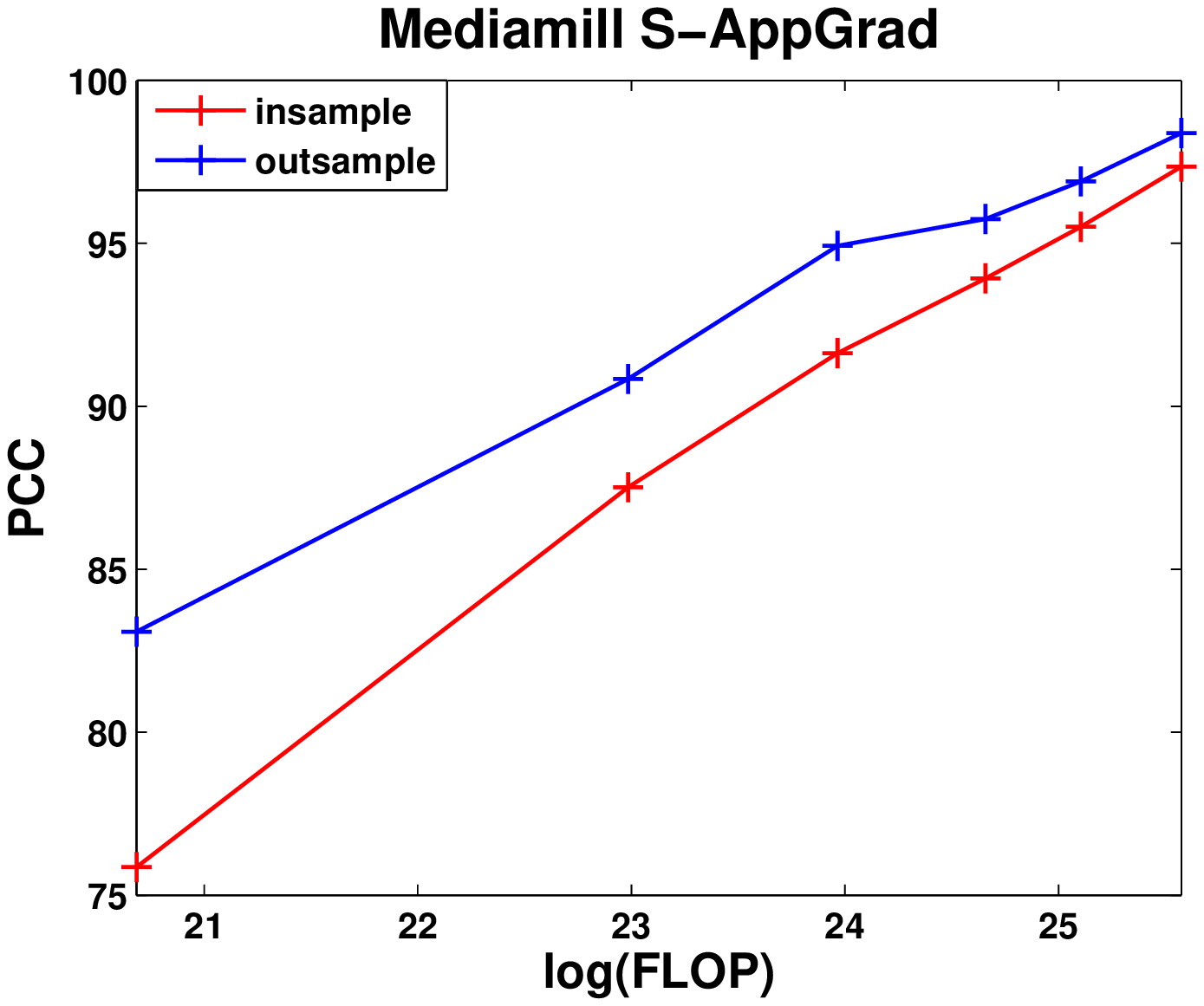}
}
\subfigure{%
\includegraphics[width=0.65\columnwidth]{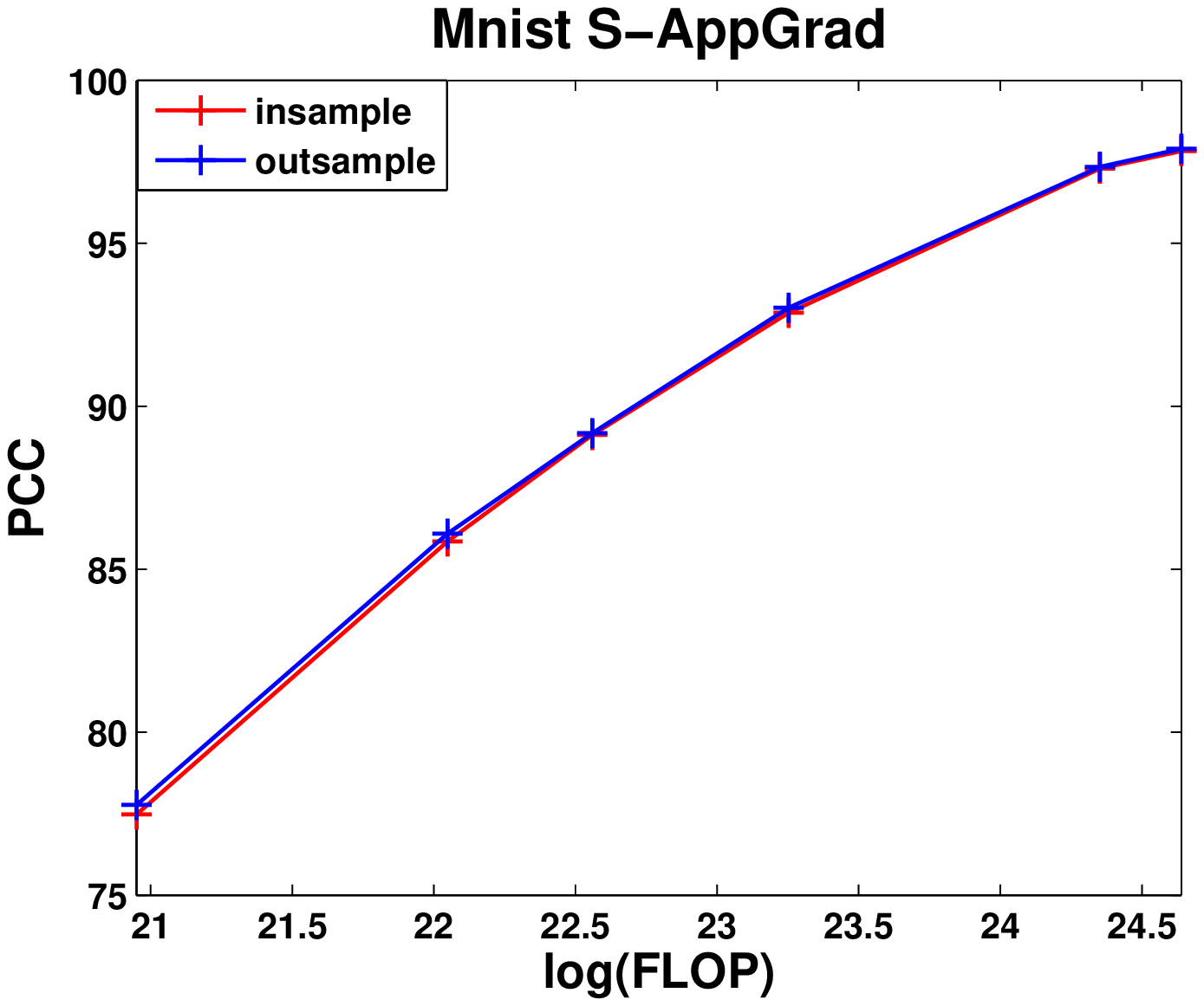}
}
\subfigure{%
\includegraphics[width=0.65\columnwidth]{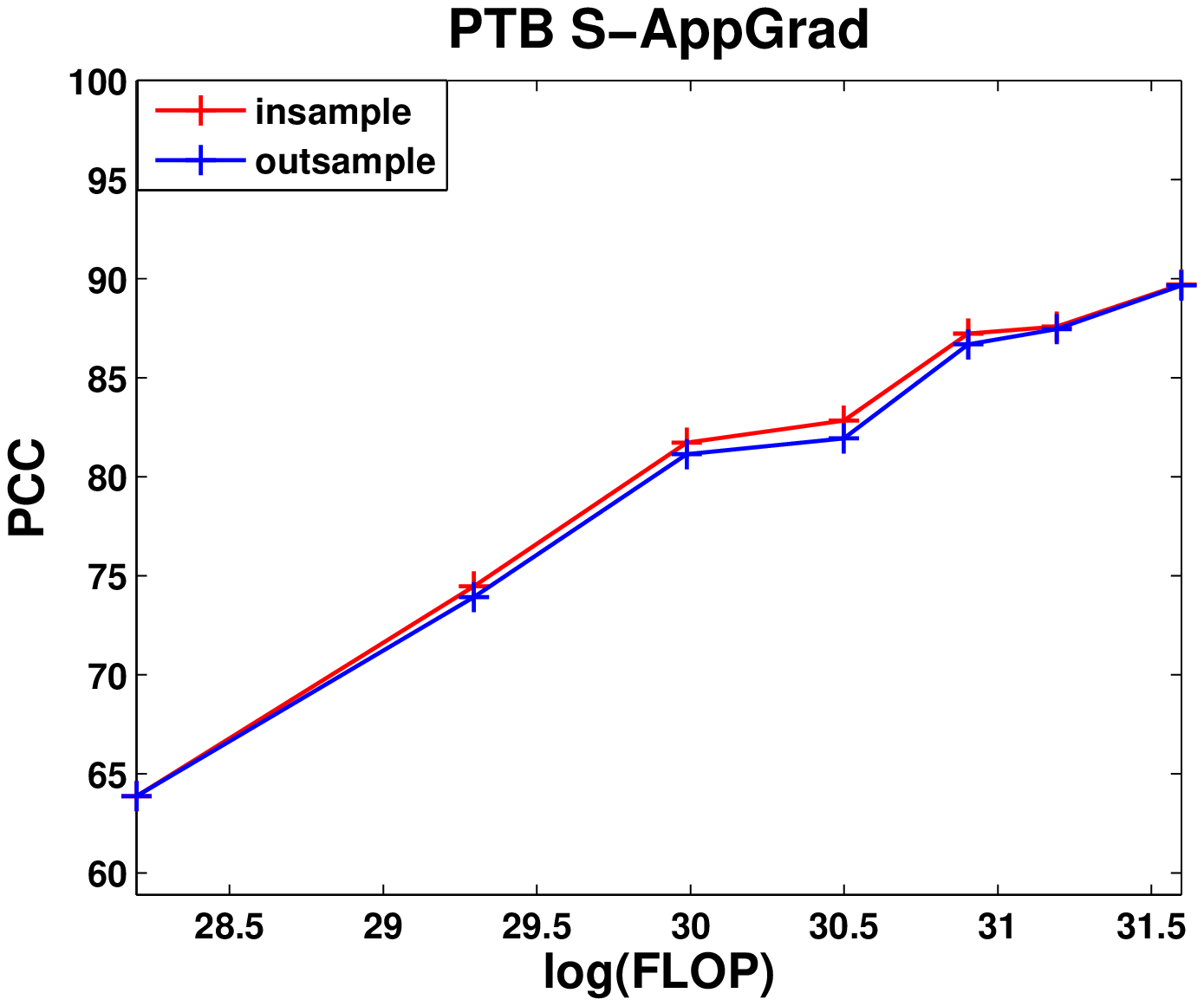}
}
%\label{•}
\caption{Proportion of Correlations Captured (PCC) by \textit{AppGrad} and stochastic \textit{AppGrad} on different datasets}
\label{fig:figure}
\end{figure*}

For URL Reputation dataset,  as we mentioned earlier, classical algorithms fails on a typical desktop. The reason is that these algorithms only produce a one-shot estimate after the whole procedure is completed, which is usually prohibitive for huge datasets.  In this scenario, the advantage of online algorithms like stochastic \textit{AppGrad} becomes crucial. Further, the stochastic nature makes the algorithm cost-effective and generate decent approximations given fixed computational resources (e.g. FLOP). As revealed by Figure~\ref{urlsgd}, as the number of iterations increases, stochastic \textit{AppGrad} captures more and more correlations.  

Since the true CCA subspaces for URL dataset is too slow to compute, we compare our algorithm with some naive heuristics which can be carried out efficiently in large scale and catches a reasonable amount of correlation. Below is a brief description of them. 
\begin{itemize}
\item Non-Whitening (NW-CCA): directly perform SVD on the unwhitened covariance matrix $\X^T\Y$. This strategy is also used in \cite{witten2009penalized}
\item Diagnoally Whitening (DW-CCA) \cite{lu2014large}: avoid inverting matrices by approximating $\sxxx, \syyy$ with $(\mbox{diag}(\sx))^{-\frac{1}{2}}$ and $(\mbox{diag}(\sy))^{-\frac{1}{2}}$. 
\item Whitening the leading $m$ Principal Component Directions (PCA-CCA): First compute the leading $m$ dimensional principal component subspace and project the data matrices $\X$ and $\Y$ to the subspace, denote them $\U_x$ and $\U_y$. Then compute the top $k$ dimensional CCA subspace of the pair $(\U_x, \U_y)$. At last, transform the CCA subspace of $(\U_x, \U_y)$ back to the CCA subspace of orginal matrix pair $(\X, \Y)$.  Specifically for this example, we choose $m=1200$ (log(FLOP)=35, dominating the computational cost of Stochastic \textit{AppGrad}) . 
\end{itemize}
\begin{remark}
For all the heuristics mentioned above,  SVD and PCA steps are carried out using the randomized algorithms in \cite{tropp11}. For PCA-CCA,  as the number of Principal Components ($m$) increases, more correlation will be captured but the computational cost will also increase. When $m=p$, PCA-CCA is reduced to the orginal CCA. 
\end{remark}
\begin{figure}
\vskip 0.2in
\begin{center}
\centerline{\includegraphics[width=0.9\columnwidth]{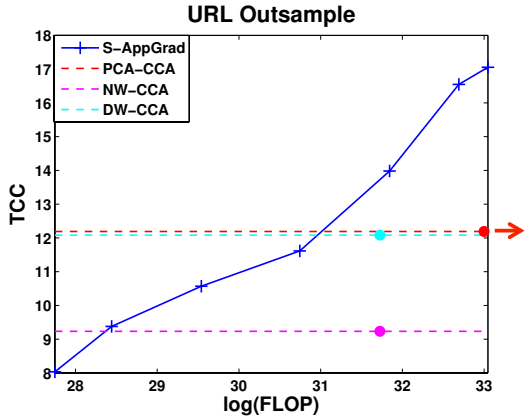}}
%\caption{}
\end{center}
\vskip -0.2in
\caption{Total Correlations Captured (TCC) by NW-CCA, DW-CCA, PCA-CCA and stochastic \textit{AppGrad} on URL dataset. The dash lines indicate TCC for those heuristics and the colored dots denote corresponding computational cost.  Red arrow means $\mbox{log(FLOP)}$ of PCA-CCA is more than 33.}
\label{urlsgd}
\end{figure} 

Essentially, all the heuristics are incorrect algorithms and try to approximately whiten the data matrices. As suggested by Figure~\ref{urlsgd},  stochastic \textit{AppGrad} significantly captures much more correlations.

%\section*{Acknowledgments} 
%\textbf{Do not} include acknowledgements in the initial version of the paper submitted for blind review.
\section{Conclusions and Future Work}
In this paper, we present a novel first order method, \textit{AppGrad}, to tackle large scale CCA as a nonconvex optimization problem. This bottleneck-free algorithm is both memory efficient and computationally scalable. More importantly, its online variant is well-suited to practical high dimensional applications where batch algorithm is prohibitive and data laden regime where data is abundant and runtime is main concern. 

Further, \textit{AppGrad} is flexible and structure information can be easily incorporated into the algorithm. For example, if the canonical vectors are assumed to be sparse \cite{witten2009penalized, gao2014sparse}, a thresholding step can be added between the gradient step and normalization step to obtain sparse solutions while it is hard to add sparse constraint to the classical CCA formulation which is a generalized eigenvalue problem. Heuristics in \cite{witten2009penalized} avoid this by simply skipping the whitening procedure (NW-CCA). \cite{gao2014sparse} resorts to semidefinite programming and therefore is very slow. \textit{AppGrad} with thresholding works well in simulations and we leave its theoretical properties for future research. 

\bibliography{cca}
\bibliographystyle{icml2015}
\twocolumn[
\icmltitle{Supplementary Material: Finding Linear Structure \\in Large Datasets with Scalable Canonical Correlation Analysis}
%\section{Supplementary Material}
\section{Proofs}
\vspace{0.1in}
\subsection{Preliminaries}
\vspace{0.1in}
A brief review of the notations in the main paper: 
$$\sx=\X^\top\X/n, \sxy=\X^\top\Y/n, \sy=\Y^\top\Y/n, \|u\|_x=(u^\top\sx u)^{\frac{1}{2}}, \|v\|_y=(v^\top\sy v)^{\frac{1}{2}}$$
$$\ttx^t=\tpx^t-\tpx_1, \tty^t=\tpy^t-\tpy_1, \tx^t=\phi^t-\phi_1, \ty^t=\psi^t-\psi_1$$
Further, we define $cos_x(u, v)=\frac{u^\top\sx v}{\xnorm{u}\,\xnorm{v}}$, the cosine of the angle between two vectors induced by the inner product $\left\langle u, v \right\rangle= u^\top\sx v$. Similarly, we define $cos_y(u, v)=\frac{u^\top\sy v}{\ynorm{u}\,\ynorm{v}}$. To prove the theorem, we will repeatedly use the following lemma. 
\begin{lemma}
\label{curve}
$\xnorm{\tx^t}\leq \frac{1}{\lambda_1}\sqrt{\frac{2}{1+cos_x(\phi^t, \phi_1)}}\xnorm{\ttx^t}$\, and \,$\ynorm{\ty^t}\leq \frac{1}{\lambda_1}\sqrt{\frac{2}{1+cos_y(\psi^t, \psi_1)}}\ynorm{\tty^t}$
\end{lemma}
\textit{Proof of Lemma~\ref{curve}}  Notice that $cos_x(\tpx^t, \tpx_1)=cos_x(\phi^t, \phi_1)$, then
\begin{equation}
\begin{aligned}
\|\ttx^t\|_x^2=\|\tpx^t-\tpx_1\|_x^2\geq \|\tpx_1\|^2 sin_x^2(\tpx^t, \tpx_1)=\lambda_1^2 sin_x^2(\phi^t, \phi_1)
\end{aligned}
\nonumber
\end{equation}
Also notice that $\|\phi^t\|_x=\|\phi_1\|_x=1$, which implies $cos_x(\phi^t, \phi_1)=1-\|\phi^t-\phi_1\|_x^2/2=1-\|\tx^t\|_x^2/2$. Further
\begin{equation}
\|\ttx^t\|_x^2\geq \lambda_1^2 sin_x^2(\phi^t, \phi_1)=\lambda_1^2(1-cos_x^2(\phi^t, \phi_1))=\frac{\lambda_1^2}{2}\|\tx^t\|_x^2(1+cos_x(\phi^t, \phi_1))
\nonumber
\end{equation}

Square root both sides, 
$$\xnorm{\tx^t}\leq \frac{1}{\lambda_1}\sqrt{\frac{2}{1+cos_x(\phi^t, \phi_1)}}\xnorm{\ttx^t}$$
Similar argument will show that 
$$\ynorm{\ty^t}\leq \frac{1}{\lambda_1}\sqrt{\frac{2}{1+cos_y(\psi^t, \psi_1)}}\ynorm{\tty^t}$$

\subsection{Proof of Theorem 2.1}
\vspace{0.1in}
Without loss of generality, we can always assume $cos_x(\tpx^t, \tpx_1), cos_y(\tpy^t, \tpy_1)\geq 0$ because the canonical vectors are only identifiable up to a flip in sign and we can always choose $\tpx_1, \tpy_1$ such that the cosines are nonnegative. Apply simple algebra to the gradient step $\tpx^{t+1}=\tpx^t-\eta(\sx\tpx^t-\sxy \psi^t)$, we have
\begin{align*}
\begin{gathered}
\tpx^{t+1}-\tpx_1=\tpx^t-\tpx_1-\eta(\sx(\tpx^t-\tpx_1)+\sx\tpx_1-\sxy (\psi^t-\psi_1)-\sxy\psi_1)\\
\ttx^{t+1}=\ttx^{t}-\eta(\sx\ttx^t-\sxy\tx^t)-\eta(\sx\tpx_1-\sxy\psi_1)
\end{gathered}
\end{align*}
By Lemma~\ref{decom}, $\eta(\sx\tpx_1-\sxy\psi_1)=\eta(\sx\tpx_1-\lambda_1\sx\phi_1)=0$, which implies
\begin{align*}
\ttx^{t+1}&=\ttx^{t}-\eta(\sx\ttx^t-\sxy\ty^t)
\end{align*}
Square both sizes, 
\begin{equation}
\|\ttx^{t+1}\|^2=\|\ttx^{t}\|^2+\eta^2\|\sx\ttx^t-\sxy\ty^t\|^2-2\eta(\ttx^{t})^\top(\sx\ttx^t-\sxy\ty^t)
\label{obj}
\end{equation}
]

\twocolumn[
Apply Lemma~\ref{decom}, 
\begin{equation}
\|\sxy\ty^t\|=\|\sx\px \lm\py^T\sy\ty^t\|\leq \|\sxx\|\|\sxx\px\|\|\lm\|\|\py^\top\syy\|\|\syy\ty^t\|\leq \lambda_1L_1^{\frac{1}{2}} \|\ty^t\|_y
\nonumber
\end{equation}
 The last inequality uses the assumption that $\lambda_{max}(\sx), \lambda_{max}(\sy)\leq L_1$. By Lemma\ref{curve}, $\|\ty^t\|_y\leq \frac{\sqrt{2}}{\lambda_1}\|\tty^t\|_y$. Hence, $\|\sxy\ty^t\|\leq \sqrt{2L_1}\|\tty^t\|_y$.
Also notice that $\|\sx\ttx^t\|\leq \|\sxx\|\|\sxx\ttx^t\|\leq L_1^\frac{1}{2}\|\ttx^t\|_x$, then
\begin{equation}
\|\sx\ttx^t-\sxy\ty^t\|^2\leq (L_1^\frac{1}{2}\|\ttx^t\|_x+ \sqrt{2}L_1^{\frac{1}{2}} \|\tty^t\|_y)^2\leq 2L_1(\|\ttx^t\|_x^2+2\|\tty^t\|_y^2)
\nonumber
\end{equation}
Substitute into \eqref{obj}, 
\begin{equation}
\|\ttx^{t+1}\|^2\leq  \|\ttx^t\|^2-2\eta\|\ttx^t\|_x^2+2L_1\eta^2(\|\ttx^t\|_x^2+2\|\tty^t\|_y^2)+2\eta(\ttx^{t})^\top\sxy\ty^t
\label{obj2}
\end{equation}
Now, we are going to bound $(\ttx^{t})^T\sxy\ty^t$. Because $\syy\py$ is an orthonormal matrix (orthogonal if $p=p_1$) and $\syy\psi_t$ is a unit vector, there exisit coefficients $\alpha_1, \cdots, \alpha_p, \alpha_\perp$ and unit vector $\psi_\perp\in ColSpan(\syy\py)^\perp$ such that $\syy\psi_t=\sum_{i=1}^p \alpha_i\syy\psi_i+\alpha_\perp\syy\psi_{\perp}, \sum_{i=1}^p \alpha_i^2+\alpha_\perp^2=1$. Therefore,
\begin{align*}
(\ttx^t)^\top\sx\px\lm\py^\top\sy\ty^t&=\ttx^t\sx\px\lm (\syy\py)^\top \{(\alpha_1-1)\syy\psi_1+\sum_{i=2}^p \alpha_i\syy\psi_i+\alpha_\perp\syy\psi_{\perp}\}\\
&=\lambda_1(\alpha_1-1)(\ttx^t)^\top\sx\phi_1+\sum_{i=2}^p \alpha_i\lambda_i(\ttx^t)^\top\sx\phi_i 
\end{align*}
By Cauthy-Schwarz inequality,
\begin{align*}
(\ttx^t)^\top\sx\px\lm\py^\top\sy\ty^t&\leq \Big(\lambda_1^2(1-\alpha_1)^2+\sum_{i=2}^p \alpha_i^2\lambda_i^2\Big)^\frac{1}{2}\Big(\sum_{i=1}^p \big((\ttx^t)^\top\sx\phi_1\big)^2\Big)^\frac{1}{2}\\
&\leq \Big(\lambda_1^2(1-\alpha_1)^2+\lambda_2^2(1-\alpha_1^2)\Big)^\frac{1}{2}\|\ttx^t\|_x\\
&=\Big(\lambda_1^2\frac{1-\alpha_1}{1+\alpha_1}+\lambda_2^2\Big)^\frac{1}{2}(1-\alpha_1^2)^{\frac{1}{2}}\|\ttx^t\|_x
\end{align*}
By definition, $1-\alpha_1=1-cos_y(\psi^t, \psi_1)=\frac{\|\ty^t\|^2_y}{2}$. Further by Lemma~\ref{curve}, 
\begin{equation}
\begin{gathered}
1-\alpha_1\leq  \frac{1}{\lambda_1^2(1+\alpha_1)}\|\tty^t\|^2_y
\end{gathered}
\nonumber
\end{equation}
Therefore, 
\begin{equation}
\begin{aligned}
(\ttx^t)^\top\sx\px\lm\py^\top\sy\ty^t&\leq \Big(\frac{1-\alpha_1}{1+\alpha_1}+\frac{\lambda_2^2}{\lambda_1^2}\Big)^\frac{1}{2} \|\ttx^t\|_x\|\tty^t\|_y\\
&\leq \Big(\frac{\|\tty^t\|_y^2}{\lambda_1^2(1+\alpha_1)^2}+\frac{\lambda_2^2}{\lambda_1^2}\Big)^\frac{1}{2} \|\ttx^t\|_x\|\tty^t\|_y\\
&\leq \frac{1}{2}\Big(\frac{\|\tty^t\|_y^2}{\lambda_1^2}+\frac{\lambda_2^2}{\lambda_1^2}\Big)^\frac{1}{2}\Big( \|\ttx^t\|_x^2+\|\tty^t\|_y^2\Big)
\end{aligned}
\nonumber
\end{equation}
Substitute into \eqref{obj2}, 
\begin{equation}
\begin{aligned}
\|\ttx^{t+1}\|^2&\leq \|\ttx^t\|^2-2\eta\|\ttx^t\|_x^2+2L_1\eta^2\Big(\|\ttx^t\|_x^2+2\|\tty^t\|_y^2\Big)+\eta\Big(\frac{\|\tty^t\|_y^2}{\lambda_1^2}+\frac{\lambda_2^2}{\lambda_1^2}\Big)^\frac{1}{2} \Big( \|\ttx^t\|_x^2+\|\tty^t\|_y^2\Big)
\end{aligned}
\nonumber
\end{equation}
Similar analysis implies that, 
\begin{equation}
\begin{aligned}
\|\tty^{t+1}\|^2&\leq \|\tty^t\|^2-2\eta\|\tty^t\|_y^2+2L_1\eta^2\Big(\|\tty^t\|_y^2+2\|\ttx^t\|_x^2\Big)+\eta\Big(\frac{\|\ttx^t\|_x^2}{\lambda_1^2}+\frac{\lambda_2^2}{\lambda_1^2}\Big)^\frac{1}{2} \Big( \|\ttx^t\|_x^2+\|\tty^t\|_y^2\Big)
\end{aligned}
\nonumber
\end{equation}
]

\twocolumn[
Add these two inequalities, 
\begin{equation}
\begin{aligned}
\|\ttx^{t+1}\|^2+\|\tty^{t+1}\|^2&\leq \Big(\|\ttx^t\|^2+\|\tty^t\|^2\Big)-2\eta\Big\{1-\frac{1}{2}\Big(\frac{\|\tty^t\|_y^2}{\lambda_1^2}+\frac{\lambda_2^2}{\lambda_1^2}\Big)^\frac{1}{2}-\frac{1}{2}\Big(\frac{\|\ttx^t\|_x^2}{\lambda_1^2}+\frac{\lambda_2^2}{\lambda_1^2}\Big)^\frac{1}{2}\\
&\quad\quad\quad-3L_1\eta\Big\}\Big(\|\ttx^t\|_x^2+\|\tty^t\|_y^2\Big) \\
\end{aligned}
\nonumber
\end{equation}
Notice that $\sqrt{a}+\sqrt{b}\leq \sqrt{2(a+b)}$, we have
\begin{equation}
\begin{aligned}
\Big(\frac{\|\tty^t\|^2_y}{\lambda_1^2}+\frac{\lambda_2^2}{\lambda_1^2}\Big)^\frac{1}{2}+\Big(\frac{\|\ttx^t\|^2_x}{\lambda_1^2}+\frac{\lambda_2^2}{\lambda_1^2}\Big)^\frac{1}{2}&\leq \Big(\frac{2\|\tty^t\|^2_y}{\lambda_1^2}+\frac{2\|\ttx^t\|^2_x}{\lambda_1^2}+\frac{4\lambda_2^2}{\lambda_1^2}\Big)^\frac{1}{2}\\
&\leq \Big(\frac{2L_1\|\tty^t\|^2}{\lambda_1^2}+\frac{2L_1\|\ttx^t\|^2}{\lambda_1^2}+\frac{4\lambda_2^2}{\lambda_1^2}\Big)^\frac{1}{2}\\
&=\frac{1}{2\lambda_1}\Big(\frac{L_1}{2}\|\tty^t\|^2+\frac{L_1}{2}\|\ttx^t\|^2+\lambda_2^2\Big)^\frac{1}{2}
\end{aligned}
\nonumber
\end{equation}
Then,
\begin{equation}
\begin{aligned}
\|\ttx^{t+1}\|^2+\|\tty^{t+1}\|^2&\leq \Big(\|\ttx^t\|^2+\|\tty^t\|^2\Big)-2\eta\Big\{1-\frac{1}{\lambda_1}\Big(\frac{L_1}{2}\|\tty^t\|^2+\frac{L_1}{2}\|\ttx^t\|^2+\lambda_2^2\Big)^\frac{1}{2}\\
&\quad\quad\quad-3L_1\eta\Big\}\Big(\|\ttx^t\|_x^2+\|\tty^t\|_y^2\Big)
\end{aligned}
\label{contraction}
\end{equation}
By definition, $\delta=1-\frac{1}{\lambda_1}\Big(\frac{L_1}{2}\|\tty^0\|^2+\frac{L_1}{2}\|\ttx^0\|^2+\lambda_2^2\Big)^\frac{1}{2}$ and $\eta=\frac{\delta}{6L_1}$. Substitute in \eqref{contraction} with $t=0$,
\begin{equation}
\begin{aligned}
\|\ttx^{1}\|^2+\|\tty^{1}\|^2&=\Big(\|\ttx^0\|^2+\|\tty^0\|^2\Big)-\frac{\delta^2}{6L_1}\Big(\|\ttx^0\|_x^2+\|\tty^0\|_y^2\Big)\\
&\leq \Big(\|\ttx^0\|^2+\|\tty^t\|^2\Big)-\frac{\delta^2}{6L_1L_2}\Big(\|\ttx^0\|^2+\|\tty^0\|^2\Big)\\
&\leq \Big(1-\frac{\delta^2}{6L_1L_2}\Big)\Big(\|\ttx^{0}\|^2+\|\tty^{0}\|^2\Big)\\
\end{aligned}
\nonumber
\end{equation}
It follows by induction that $\forall\, t\in\mathbb{N}_+$
\begin{equation}
\|\ttx^{t+1}\|^2+\|\tty^{t+1}\|^2\leq \Big(1-\frac{\delta^2}{6L_1L_2}\Big)\Big(\|\ttx^{t}\|^2+\|\tty^{t}\|^2\Big)
\nonumber
\end{equation}
\subsection{Proof of Proposition 2.3}
\vspace{0.1in}
Substitute $(\px^t,\py^t, \widetilde{\px}^t, \widetilde{\py}^t)=(\px_k, \py_k, \px_k\lm_k, \py_k\lm_k)\Q$ into the iterative formula in Algorithm 4. 
\begin{equation}
\begin{aligned}
\widetilde{\px}^{t+1}&=\px_k\lm_k\Q-\eta_1(\sx\px_k\lm_k-\sxy\py_k)\Q\\
&=\px_k\lm_k\Q-\eta_1(\sx\px_k\lm_k-\sx\px \lm\py^\top\sy\py_k)\Q\\
&=\px_k\lm_k\Q-\eta_1(\sx\px_k\lm_k-\sx\px_k\lm_k)\Q\\
&=\px_k\lm_k\Q
\end{aligned}
\nonumber
\end{equation}
The second equality is direct application of Lemma 1. The third equality is due to the fact that $\py^\top\sy\py=I_p$. Then,
$$(\widetilde{\px}^{t+1})^\top\sx\widetilde{\px}^{t+1}=\Q^\top\lm_k^2\Q$$
and
$$\px^{t+1}=\widetilde{\px}^{t+1}\Q^\top\lm_k^{-1}\Q=\px_k\Q$$
Therefore $(\px^{t+1}, \widetilde{\px}^{t+1})=(\px^{t}, \widetilde{\px}^{t})=(\px_k, \px_k\lm_k)\Q$.
A symmetric argument will show that $(\py^{t+1}, \widetilde{\py}^{t+1})=(\py^{t}, \widetilde{\py}^{t})=(\py_k, \py_k\lm_k)\Q$, which completes the proof. 
]
\end{document}